\documentclass[10pt,twocolumn,letterpaper]{article}
\usepackage{iccv}

\usepackage{times}
\usepackage{graphicx}
\usepackage{amsmath}
\usepackage{amssymb}
\usepackage{amsthm}
\newtheorem{prop}{Proposition}

\newtheorem{lemma}{Lemma}
\newtheorem{theorem}{Theorem}

\usepackage{epsfig}
\usepackage{float}
\usepackage{graphics}
\usepackage{algorithmic}
\DeclareMathOperator*{\argmin}{arg\;min}
\DeclareMathOperator*{\argmax}{arg\;max}

\usepackage[pagebackref=true,breaklinks=true,letterpaper=true,colorlinks,bookmarks=false]{hyperref}

\iccvfinalcopy 

\def\iccvPaperID{4113} 

\ificcvfinal\fi

\begin{document}

\title{Unsupervised Domain Adaptation via Regularized Conditional Alignment}

\author{Safa Cicek, Stefano Soatto \\
UCLA Vision Lab\\
University of California, Los Angeles, CA 90095\\
{\tt\small \{safacicek,soatto\}@ucla.edu}
}

\maketitle

\begin{abstract}
We propose a method for unsupervised domain adaptation that trains a shared embedding to align the joint distributions of inputs (domain) and outputs (classes), making any classifier agnostic to the domain. Joint alignment ensures that not only the marginal distributions of the domain are aligned, but the labels as well. We propose a novel objective function that encourages the class-conditional distributions to have disjoint support in feature space. We further exploit adversarial regularization to improve the performance of the classifier on the domain for which no annotated data is available.
\end{abstract}

\section{Introduction}

In the context of classification, unsupervised domain adaptation (UDA) consists of modifying a classifier trained on a labeled dataset, called the ``source,'' so it can function on data from a different ``target'' domain, for which no annotations are available. More in general, we want to train a model to operate on input data from both the source and target domains, despite absence of annotated data for the latter. For instance, one may have a synthetic dataset, where annotation comes for free, but wish for the resulting model to work well on real data, where manual annotation is scarce or absent \cite{richter2016playing}.  

\begin{figure}[t]
\centering
\includegraphics[width=9cm]{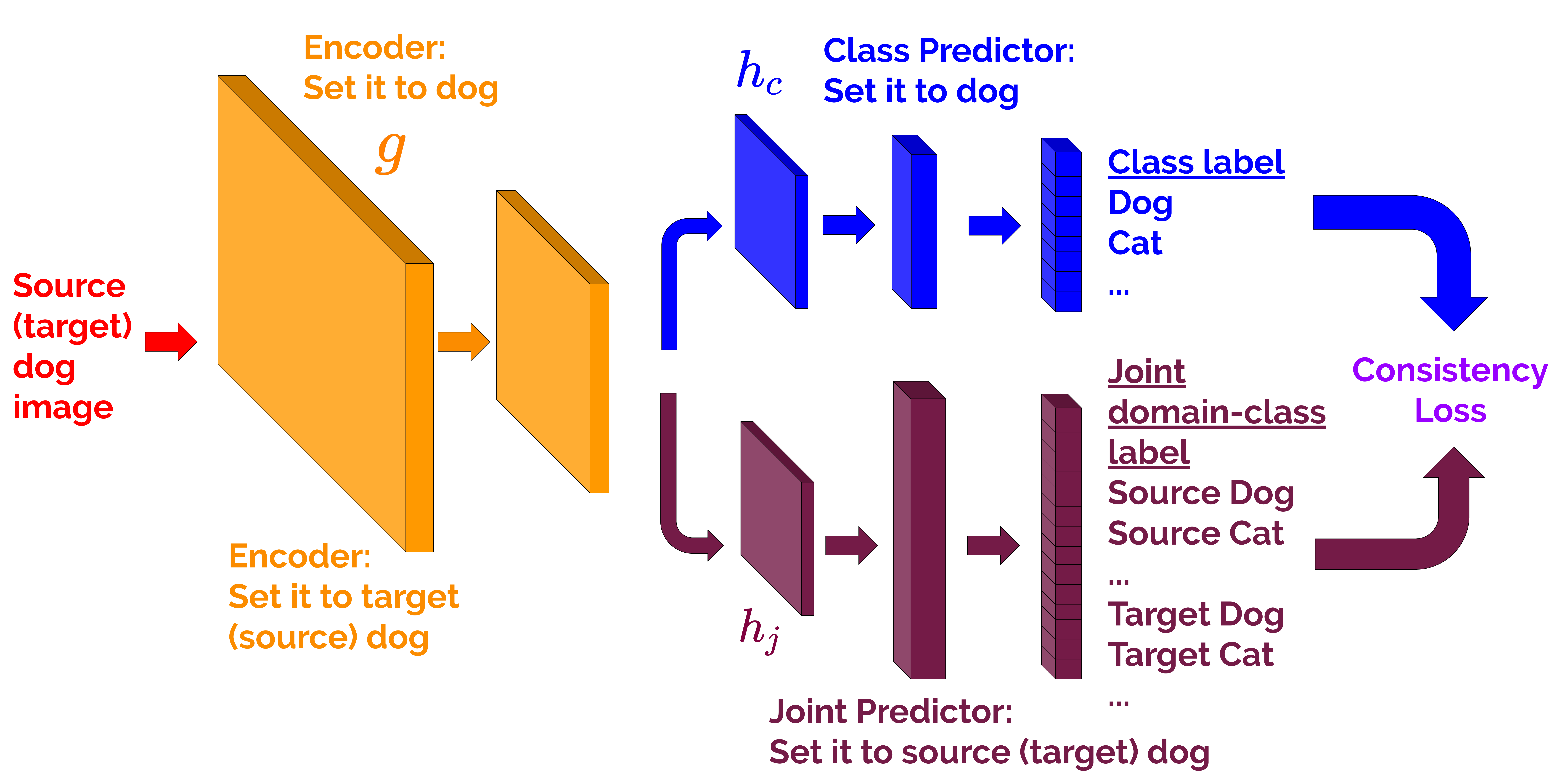}
\caption{The network structure of the proposed approach. We propose to learn a joint distribution $P(d,y)$ over domain label $d$ and class label $y$ by a joint predictor (purple). The encoder (orange) is trained to confuse this joint predictor by matching the features corresponding to the same category samples of both domains. Since labels for the target data is not known, predictions of the class predictor (blue) on the target data is used with the help of consistency loss. Unlabeled data is further exploited with input smoothing algorithm VAT \cite{miyato2017virtual} from the SSL literature.}
\label{cda_network}
\end{figure}

The most successful methods learn the parameters of a deep neural network using adversarial (min-max) criteria. The idea is to simultaneously recognize the class (output) as well as the domain ({\em e.g.}, ``real vs. synthetic'') by training the classifier to work as well as possible on both while encoder is fooling the discriminator for the latter. In a sense, the classifier becomes agnostic to the domain. This can be understood as aligning the marginal distribution of the inputs from the two domains. Unfortunately, this does not guarantee successful transfer, for it is possible that the source (say synthetic images) be perfectly aligned with the target (say natural images), and yet a natural image of a cat map to a synthetic image of a dog. It would be desirable, therefore, for the adaptation to align the outputs, along with the inputs. This prompted other methods to align, instead of the marginal distributions, the join or conditional distribution of domain and class. This creates two problems: First, the target class labels are unknown; second, since there is a shared representation of the inputs, aligning the joint distributions may cause them to collapse thus losing the discriminative power of the model. 

To address these problems, we propose a method to perform the alignment of the joint distribution (Sect. \ref{section_loss}). We employ ideas from semi-supervised learning (SSL) to improve generalization performance (Sect. \ref{section_exploiting}). We propose an optimization scheme that uses a two-folded label space. The resulting method performs at the state of the art without pushing the limits of hyperparameter optimization (Sect. \ref{section_empirical}). We analyze the proposed objective function in the supervised setting and prove that the optimal solution conditionally aligns the distributions while keeping them discriminative (Sect. \ref{section_analysis}). Finally, we discuss our contribution in relation to the vast and growing literature on UDA (Sect. \ref{section_discussion}).

\subsection*{Formalization}

We are given $N^s$ labeled source samples $x^s \in X^s$ with corresponding labels $y^s \in Y^s$ and $N^t$ unlabeled target samples, $x^t \in X^t$. The entire training dataset $X$ has cardinality $N = N^s + N^t$. Labeled source data and unlabeled target data are drawn from two different distributions (domain shift): $(x^s, y^s) \sim P^s$, $(x^t, y^t) \sim P^t$ where their discrepancy, measured by Kullbach-Liebler's (KL) divergence, is $KL(P^s || P^t) > 0$ (covariate shift). Both distributions are defined on $X \times Y$ where $Y=\{1,...,K\}$. Marginal distributions are defined on $X$ and samples are drawn from them as $x^s \sim P^s_x$, $x^t \sim P^t_x$. Given finite samples $\{(x^s_i, y^s_i)\}_{i=1}^{N^s} := \{(x^s_1,y^s_1),(x^s_2,y^s_2),...,(x^s_{N^s},y^s_{N^s})\}$ from $P^s$ and $ \{(x^t_i)\}_{i=1}^{N^t} := \{x^t_1,x^t_2,...,x^t_{N^t}\}$ from $P^t_x$, the goal is to learn a classifier $f: X \rightarrow Y$ with a small risk in the target domain. 
This risk can be measured with cross-entropy: \begin{align}
\min_f  &\  E_{(x,y) \sim P^t} \ell_{CE}(f(x); y)
\end{align} where \begin{equation}
\ell_{CE}(f(x); y) := - \langle y, \log f(x)\rangle
\label{ce_loss}
\end{equation} is the cross-entropy loss calculated for one-hot, ground-truth labels $y \in \{0,1\}^K$ and label estimates $f(x) \in \mathbb R^K$, which is the output of a deep neural network with input $x$, and $K$ is the number of classes.

\section{Proposed Method}

In this section, we show how to formalize the criterion for aligning both inputs and outputs, despite the latter being unknown for the target classes in the absence of supervision.

Alignment of the marginal distributions can be done using Domain Adversarial Neural Networks (DANN) \cite{ganin2014unsupervised}, that add to the standard classification loss for the source data a binary classification loss for the domain: Source vs. target. If all goes well, the class predictor classifies the {\em source} data correctly, and the binary-domain predictor is unable to tell the difference between the source and the target data. Therefore, the class predictor might also classify the target data correctly. Unfortunately, this is not guaranteed as there can be a misalignment of the output spaces that cause some class in the source to map to a different class in the target, {\em e.g.}, a natural cat to a synthetic dog.

The {\em key idea} of our approach is to impose {\em not} a binary adversarial loss on the domain alignment, but a $2K$-way adversarial loss, as if we had $2K$ possible classes: The first $K$ are the known \textit{source classes}, and the second $K$ are the unknown \textit{target classes}. We call the result a {\em joint domain-class predictor} or {\em joint predictor} in short, since it learns a distribution over domain and class variables. The encoder will try to fool the predictor by minimizing the classification loss between a dog sample in the source and a sample in the target domain whose predicted label in the aligned domain is also dog.

During training, the probabilities assigned to the first $K$ labels of the joint predictor are very small for the target samples, and they eventually converge to zero. Therefore, we need a separate mechanism to provide pseudo-labels to the target samples to be aligned by the joint predictor. For this, we train another predictor, that we call {\em class predictor} outputting only class labels. The class predictor is trained on both the source data using ground-truth labels and the target data using semi-supervised learning (SSL) regularizers.  

Both the joint predictor and the class predictor can be used for inference. However, we find that the class predictor performs slightly better. We conjecture this is because joint predictor is trained on a harder task of domain and class prediction while only the latter one is needed at inference time. 

We consider UDA as a two-fold problem. The first step deals with domain shift by aligning distributions in feature space. Given a successful alignment, one can use a source-only trained model for inference. But, once the domains are matched, it is possible to further improve generalization by acting on the label space. Ideas from SSL can help to that end \cite{miyato2017virtual}. 

The overall architecture of the model is described in Fig. \ref{cda_network}.

\subsection{Network structure}

We denote the shared encoder with $g$, the class predictor with $h_c$, the joint predictor with $h_j$ and the overall networks as $f_c=h_c \circ g$ and $f_j=h_j \circ g$. Then, the class-predictor output for an input $x$ can be written as, \begin{align}
f_c(x) = h_c(g(x)) \in \mathbb{R}^{K}. 
\end{align} Similarly, the joint-predictor output can be written as, \begin{align}
f_j(x) = h_j(g(x)) \in \mathbb{R}^{2K}.
\end{align} 

\subsection{Loss functions}
\label{section_loss}
The class predictor is the main component of the network which is used for inference. Its marginal features are aligned by the loss provided by the joint predictor. The class predictor is trained with the labeled source samples using the cross-entropy loss. This source classification loss can be written as,
\begin{align}
L_{sc}(f_c) = E_{(x,y) \sim P^s} \ell_{CE}(f_c(x), y).
\label{loss_sc}
\end{align} Both the encoder ($g$) and the class predictor $(h_{c})$ are updated while minimizing this loss. 

We also update the joint predictor with the same classification loss for the labeled source samples. This time, only the joint-predictor ($h_j$) is updated. The joint-source classification loss is
\begin{align}
L_{jsc}(h_j) = E_{(x,y) \sim P^s} \ell_{CE}(h_j(g(x)), [y,\, \mathbf{0}])
\label{loss_jsc}
\end{align} where $\mathbf{0}$ is the zero vector of size $K$, chosen to make the last $K$ joint probabilities zero for the source samples. 

Similarly, the joint predictor is trained with target samples. As ground-truth labels for the target samples are not given, label estimates from the class predictors are used as pseudo-labels. The joint target classification loss is \begin{align}
L_{jtc}(h_j) = E_{x \sim P^t_x} \ell_{CE}(h_j(g(x)), [\mathbf{0},\, \hat y] )
\label{loss_jtc}
\end{align} where $\hat y = e_k$ and $k = \argmax_{k} f_c(x)[k] = \argmax_{k} h_c(g(x))[k]$, $e_k$ is the identity of size $K$ whose $k$th element is $1$.\footnote{We use the notation $x[k]$ for indexing the value at the $k$th index of the vector $x$.} Here, we assume that the source-only model achieves reasonable performance on the target domain (e.g. better than a chance). For experiments where the source-only trained model has poor performance initially, we apply this loss after the class predictor is trained for some time. Since the joint predictor is trained with the estimates of the class predictor on the target data, it can also be interpreted as a {\em student} of the class predictor. 

The goal of introducing a joint predictor was to align label-conditioned feature distributions. For this, encoders are trained to fool the joint predictor as in \cite{ganin2014unsupervised}. Here, we apply conditional fooling. The joint source alignment loss is
\begin{align}
\label{loss_jsa}
L_{jsa}(g) = E_{(x,y) \sim P^s} \ell_{CE}(h_j(g(x)), [\mathbf{0},\, y]).
\end{align} 
The encoder is trained to fool by changing the joint label from $[y, \, \mathbf{0}]$ to $[\mathbf{0},\,  y]$. Similary, the joint-target alignment loss is defined by changing the pseudo-labels from $[\mathbf{0},\, \hat y]$ to $[\hat y,\,  \mathbf{0}]$,
\begin{align}
L_{jta}(g) = E_{x \sim P^t_x} \ell_{CE}(h_j(g(x)), [\hat y,\,  \mathbf{0}]).
\label{loss_lcta}
\end{align} The last two losses are minimized only by the encoder $g$.

\subsection{Exploiting unlabeled data with SSL regularizers}
\label{section_exploiting}

\begin{figure}[t]
\centering
\includegraphics[width=9cm]{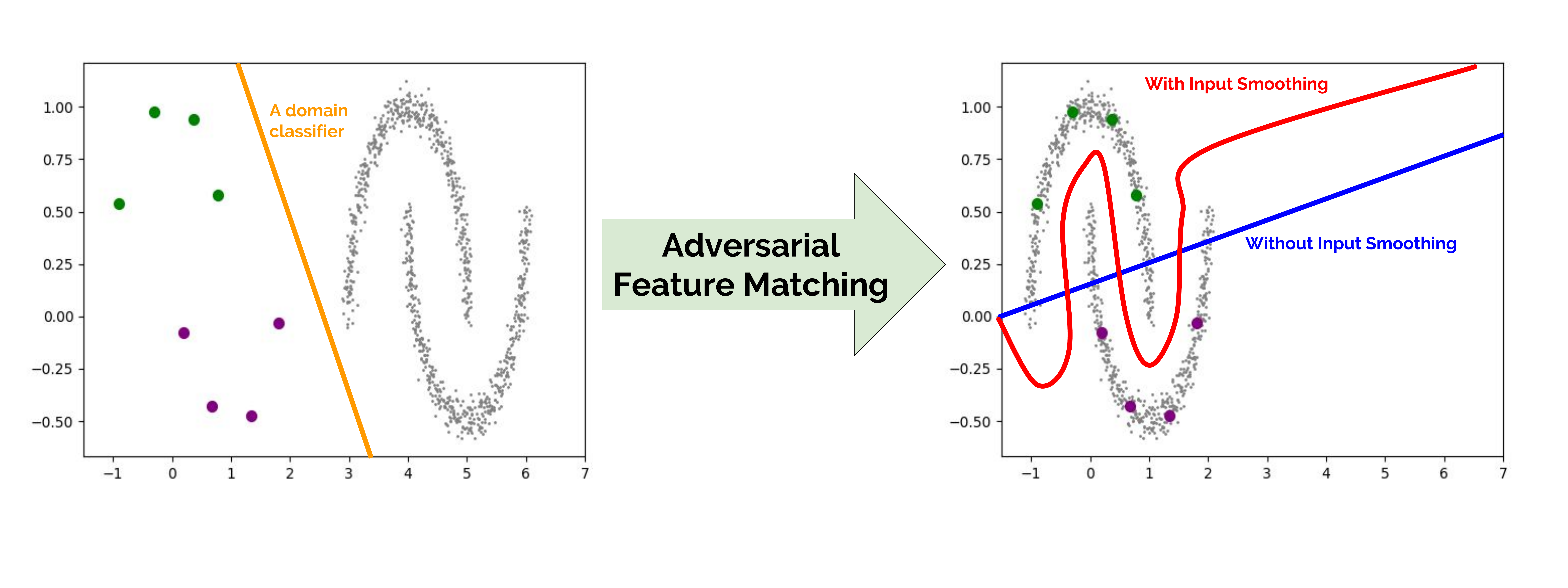}
\caption{{\bf Left.} In the UDA setting, there exist domain classifiers (e.g. orange line segment) being able to distinguish the source samples (green and purple dots) from the target samples (gray dots). Conditional feature matching is applied until there is no such classifier in the finite-capacity classifier space. As a result, the label-conditioned feature distributions of the source and the target data are matched. {\bf Right.} Once the features are matched, exploiting unlabeled data using SSL regularizers like VAT \cite{miyato2017virtual} becomes trivial. Only using labeled samples (green and purple dots) gives a poor decision boundary (blue line segment). When input adversarial training is applied using unlabeled samples (gray dots), desired decision boundary is achieved (red curve). Best viewed in color.}
\label{ssl_uda}
\end{figure} 

Once features of the source and the target domains are matched, our formulation of UDA turns into a semi-supervised learning problem. In a way, adversarial domain adaptation deals with the large domain shift between source and target datasets while adversarial input smoothing removes the shift in predictions within a small neighborhood of a domain (See Fig. \ref{ssl_uda}).

For a discriminative model to exploit unlabeled data, there has to be some prior on the model parameters or on the unknown labels \cite{chapelle2009semi}. Applying entropy minimization for the predictions on the unlabeled data is a well-known regularizer in the SSL literature \cite{grandvalet2005semi,krause2010discriminative,cicek2018saas}. This regularization forces decision boundaries to be in the low-density region, a desired property under the cluster assumption \cite{chapelle2009semi}. Our class predictor is trained to minimize this target entropy loss, \begin{equation}
L_{te}(f_c) = E_{x \sim P^t_x} \ell_{E}(h_c(g(x))) 
\label{ent_loss}
\end{equation} where $\ell_{E}(f(x)) := - \langle f(x), \log f(x)\rangle$. Since the joint predictor is already trained on the low-entropy estimates of the class predictor, it is enough to apply it to the class predictor.  
Minimizing entropy satisfies the cluster assumption only for Lipschitz classifiers \cite{grandvalet2005semi}. The Lipschitz condition can be realized by applying adversarial training as suggested by \cite{miyato2015distributional,miyato2017virtual}. VAT \cite{miyato2017virtual} makes a second-order approximation for adversarial input perturbations $\Delta x$ and proposes the following approximation to the adversarial noise for each input $x$: \begin{align}
& \Delta x \approx \epsilon_x \frac{r}{||r||_2} \nonumber \\
{\rm subject \ to} &\ r =  \nabla_{\Delta x} \ell_{CE}( f(x), f(x+\Delta x) ) \Big|_{\Delta x = \xi d}
\label{vat_eq}
\end{align} where $d \sim N(0,1)$. Therefore, the regularization loss of \cite{miyato2015distributional,miyato2017virtual} is
\begin{align}
\ell_{VAT}(f(x)) := \ell_{CE}(f(x), f(x+\epsilon_x \frac{r}{||r||_2}) ) \nonumber  \\
{\rm subject \ to \ } r =  \nabla_{\Delta x} \ell_{CE}( f(x), f(x+\Delta x) ) \Big|_{\Delta x = \xi d}
\label{vat_loss}
\end{align} for one input sample $x$. We will apply this regularizer both on the source and the target training data as in \cite{shu2018dirt,kumar2018co}. So, the source and target losses are given as follows:
\begin{align}
L_{svat}(f_c) = E_{(x,y) \sim P^s} \ell_{VAT}(f_c(x))
\end{align} and
\begin{align}
L_{tvat}(f_c) = E_{x \sim P^t_x} \ell_{VAT}(f_c(x)).
\end{align} SSL regularizations can be applied in a later stage once feature matching is achieved \cite{shu2018dirt}. But, we find that in most tasks, applying SSL regularizers from the beginning of the training also works well. More details are given in the Supp. Mat.

We combine the objective functions introduced in this section and in the previous section. The overall adversarial loss functions for the source and the target samples can be written as follows,
\begin{align}
L_{adv}(g) = \lambda_{jsa} L_{jsa}(g) + \lambda_{jta} L_{jta}(g) 
\label{loss_adv}
\end{align}
The remaining objective functions are
\begin{align}
L(g, h_{j}, h_{c}) = L_{s}(g, h_{j}, h_{c}) + \lambda_t L_{t}(g, h_{j}, h_{c}). 
\label{loss_all}
\end{align}
where
\begin{align}
L_{s}(g, h_{j}, h_{c}) = L_{sc}(f_c) + \lambda_{svat} L_{svat}(f_c) + \nonumber\\ 
\lambda_{jsc} L_{jsc}(h_{j})
\end{align}
\begin{align}
L_{t}(g, h_{j}, h_{c}) =  L_{te}(f_c) + \lambda_{tvat} L_{tvat}(f_c) +\nonumber\\
\lambda_{jtc} L_{jtc}(h_{j}).
\end{align}
The proposed method minimizes Eq. \ref{loss_adv} and Eq. \ref{loss_all} in an alternating fashion.

\subsection{Connection to domain adaptation theory}
The work of \cite{ben2010theory} provides an upper bound on the target risk: $\ell_t(h,y) = E_{(x, y) \sim P^t} [|h(x)-y|]$ where $h$ is the classifier. One component in the upper bound is  a divergence term between two domain distributions. In UDA, we are interested in the difference of the measures between subsets of two domains on which a hypothesis in the finite-capacity hypothesis space $\mathcal H$ can commit errors. Instead of employing traditional metrics (e.g. the total variation distance), they use the $\mathcal H$-divergence. Given a domain $X$ with $P$ and $Q$ probability distributions over $X$, and $\mathcal H$ a hypothesis class on $X$, the $\mathcal H$-divergence is
\begin{align}
d_{\mathcal H \Delta \mathcal H}(P, Q) := 2 \sup_{h, h' \in \mathcal H} |Pr_{x\sim P}(h(x) \neq h'(x)) -\nonumber\\
Pr_{x\sim Q}(h(x) \neq h'(x))|.
\end{align} 

\begin{table*}
\begin{footnotesize}
\begin{center}
\begin{tabular}{||c | c | c | c | c | c||} 
\hline
Dataset & Number of training samples & Number of test samples & Number of classes & Resolution & Channels \\ [0.5ex]
\hline\hline
\hline
MNIST \cite{lecun1998gradient} & $60,000$ & $10,000$ & $10$ & $28 \times 28$ & Mono\\
\hline
SVHN \cite{netzer2011reading} & $73,257$ & $26,032$ & $10$ & $32 \times 32$ & RGB\\
\hline
CIFAR10 \cite{krizhevsky2009learning} & $50,000$ & $10,000$ & $10$ & $32 \times 32$ & RGB \\
\hline
STL  \cite{coates2011analysis}& $5,000$ & $8,000$ & $10$ & $96 \times 96$ & RGB \\
\hline
SYN-DIGITS \cite{ganin2014unsupervised} & $479,400$ & $9,553$ & $10$ &  $32 \times 32$ & RGB \\
\hline\hline
\end{tabular}
\caption{Specs of the datasets used in the experiments.}
\label{uda_datasets}
\end{center}
\end{footnotesize}
\end{table*}

Now, we can recall the main Theorem of \cite{ben2010theory}. \textit{Let $\mathcal H$ be an hypothesis space of VC dimension $d$. If $X^s$, $X^t$ are unlabeled samples of size $m'$ each, drawn from $P_x^s$ and $P_x^t$ respectively, then for any $\delta \in (0,1)$, with probability at least $1-\delta$ (over the choice of the samples), for every $h\in \mathcal H$: \begin{align}
\ell_t(h) \leq \ell_s(h) + \frac{1}{2} \hat d_{H \Delta H} (X^s, X^t) +\nonumber \\
4 \sqrt{\frac{2d \log(2m')+\log(\frac{2}{\delta})}{m'}} + \lambda
\end{align} where $\lambda=\ell_s(h^*)+\ell_t(h^*)$, $h^*=\argmin_{h\in \mathcal H} \ell_s(h)+\ell_t(h)$ and $\hat d_{H \Delta H} (X^s, X^t)$ is empirical $\mathcal H$ divergence.}  In words, the target risk is upper bounded by the source risk, empirical $\mathcal H$-divergence and combined risk of ideal joint hypothesis $\lambda$. 

If there is no classifier which can discriminate source samples from target samples then the empirical $\mathcal H$-divergence is zero from Lemma 2 of \cite{ben2010theory}. DANN of \cite{ganin2014unsupervised} minimizes $\hat d_{H \Delta H} (X^s, X^t)$ by matching the marginal distributions (i.e. by aligning marginal push-forwards $g\#P^s_x$ and $g\#P^t_x$). But, if the joint push-forward distributions ($g\#P^s$ and $g\#P^t$) are not matched accurately, there may not be a classifier in the hypothesis space with low risk in both domains. Hence, $\lambda$ has to be large for any hypothesis space $\mathcal H$.

Our proposed method tackles this problem, by making sure that the label-conditioned push-forwards are aligned disjointly. With disjoint alignment, we mean that no two samples with different labels can be assigned to the same feature point. Moreover, the third term in the upper bound decreases with the number of samples drawn from both domains. This number can increase with data augmentation. VAT has the same effect of augmenting the data with adversarially perturbed images where the small perturbations are nuisances for the task. 

\section{Empirical Evaluation}
\label{section_empirical}
\subsection{Implementation details}
We evaluate the proposed method on the standard digit and object image classification benchmarks in UDA. Namely, CIFAR $\rightarrow$ STL, STL $\rightarrow$ CIFAR, MNIST $\rightarrow$ SVHN, SVHN $\rightarrow$ MNIST, SYN-DIGITS $\rightarrow$ SVHN and MNIST $\rightarrow$ MNIST-M.  The first three settings are the most challenging ones where state-of-the-art (SOA) methods accuracies are still below $90\%$. Our method achieves SOA accuracy in all these tasks.

\textbf{CIFAR $\leftrightarrow$ STL.} Similar to CIFAR, STL images are acquired from labeled examples on ImageNet. However, images are $96\times 96$ instead of the $32\times 32$ images in CIFAR. All images are converted to $32\times 32$ RGB in pretraining. We down-sampled images by local averaging. Note that we only used the labeled part of STL in all the experiments. CIFAR and STL both have $10$ classes, $9$ of which are common for both datasets. Like previous works \cite{kumar2018co,french2018self} we removed the non-overlapping classes (class frog and class monkey) reducing the problem into a 9-class prediction. See Table \ref{uda_datasets} for the specs of the datasets.

\textbf{MNIST $\leftrightarrow$ SVHN.} We convert MNIST images to RGB images by repeating the gray image for each color channel and we resize them to $32\times 32$ by padding zeros. Following previous works \cite{kumar2018co,french2018self}, we used Instance Normalization (IN) for MNIST $\leftrightarrow$ SVHN, which is introduced by \cite{ulyanov2017improved} for image style transfer. We preprocess images both at training and test time with IN.  

\textbf{SYN-DIGITS $\rightarrow$ SVHN.} SYN-DIGITS \cite{ganin2014unsupervised} is a dataset of synthetic digits generated from Windows fonts by varying position, orientation and background. In each image, one, two or three digits exist. The degrees of variation were chosen to match SVHN.

\textbf{MNIST $\rightarrow$ MNIST-M.} MNIST-M \cite{ganin2014unsupervised} is a difference-blend of MNIST over patches randomly extracted from color photos from BSDS500 \cite{arbelaez2011contour}. I.e. $I_{ijk}^{out} = |I_{ijk}^1-I_{ijk}^2|$, where $i, j$ are the coordinates of a pixel and $k$ is a channel index. MNIST-M images are RGB and $28$ by $28$. MNIST images are replicated for each channel during preprocessing.

No data augmentation is used in any of the experiments to allow for a fair comparison with SOA methods \cite{kumar2018co,shu2018dirt}. Again, to allow fair comparison with the previous works \cite{french2018self,shu2018dirt}, we have not used sophisticated architectures like ResNet \cite{he2016deep}. Networks used in the experiments are given in the Supp. Mat. We report inference performance of the class-predictor.

We feed source and training samples into two different mini-batches at each iteration of training. As we are using the same batch layers for both source and target datasets, mean and variance learned -- to be used at inference time -- are the running average over both source and target data statistics. 

Office UDA experiments (Amazon$\rightarrow$Webcam,  Webcam$\rightarrow$DSLR, DSLR$\rightarrow$Webcam) were used as the standard benchmark in early UDA works \cite{long2016unsupervised,shen2017adversarial,li2016revisiting,tzeng2017adversarial}. However, recent SOA methods \cite{saito2017asymmetric,kumar2018co,shu2018dirt,french2018self} did not report on these datasets, as labels are noisy \cite{bousmalis2017unsupervised}. Moreover, this is a small dataset with $4,652$ images from $31$ classes necessitating the use of Imagenet-pretrained networks. Hence, we also choose not to report experiments on this dataset.

\subsection{Results}
\label{sect_results}
\begin{table*}[t]
\begin{center}
\begin{tabular}{||c | c | c | c | c | c | c ||} 
\hline
Source dataset & MNIST  & SVHN  & CIFAR  & STL   & SYN-DIGITS & MNIST\\ [0.5ex]
Target dataset & SVHN & MNIST & STL & CIFAR  & SVHN & MNIST-M\\ [0.5ex]
\hline\hline
\cite{ganin2014unsupervised} DANN* & $60.6$ & $68.3$ & $78.1$ &  $62.7$ & $90.1$ & $94.6$\\
\hline
\cite{ghifary2016deep} DRCN & $40.05$ & $82.0$ & $66.37$ &  $58.86$ & NR& NR\\
\hline
\cite{sener2016learning} kNN-Ad & $40.3$  & $78.8$ & NR &  NR & NR & $86.7$\\
\hline
\cite{saito2017asymmetric} ATT & $52.8$ & $86.2$ & NR &  NR & $92.9$ & $94.2$\\
\hline
\cite{french2018self} $\Pi$-model** & $33.87$ & $93.33$ & $77.53$ & $71.65$& $96.01$& NR\\
\hline
\cite{shu2018dirt} VADA & $47.5$ & $97.9$ & $80.0$ & $73.5$ & $94.8$& $97.7$\\  
\hline
\cite{shu2018dirt} DIRT-T & $54.5$ & $99.4$ & NR & $75.3$  & $96.1$& $98.9$\\
\hline
\cite{shu2018dirt} VADA + IN & $73.3$  & $94.5$ & $78.3$ & $71.4$   & $94.9$& $95.7$\\  
\hline
\cite{shu2018dirt} DIRT-T +IN & $76.5$ & $99.4$ & NR & $73.3$ & $96.2$& $98.7$\\
\hline
\cite{kumar2018co} Co-DA & $81.7$ & $99.0$& $81.4$ & $76.4$ & $96.4$& $99.0$\\
\hline
\cite{kumar2018co} Co-DA + DIRT-T & $88.0$ & $\boldsymbol{99.4}$ & NR & $77.6$ & $\boldsymbol{96.4}$ & $99.1$\\
\hline\hline
Ours & $\boldsymbol{89.19}$ & $99.33$ & $\boldsymbol{81.65}$ & $\boldsymbol{77.76}$ & $96.22$& $\boldsymbol{99.47}$\\
\hline\hline
Source-only (baseline) & $44.21$ & $70.58$ & $79.41$ & $65.44$ & $85.83$ & $70.28$\\
\hline
Target-only & $94.82$ & $99.28$ & $77.02$  & $92.04$ & $96.56$ & $99.87$\\
\hline
\end{tabular}
\caption{{\bf Comparison to SOA UDA algorithms on the UDA image classification tasks.} Accuracies on the target test data are reported. Algorithms are trained on entire labeled source training data and unlabeled target training data. NR stands for not reported. * DANN results are implementation of \cite{shu2018dirt} with instance normalized input. ** Results of \cite{french2018self} with minimal augmentations are reported. The proposed method achieves the best or second highest score after Co-DA. The proposed method can be combined with Co-DA, but we report the naked results to illustrate the effectiveness of the idea.}
\label{soa_uda}
\end{center}
\end{table*}

We report the performance of the proposed method in Table \ref{soa_uda}. In all the experiments, the proposed method achieves the best or the second best results after Co-DA \cite{kumar2018co}. Especially in the most challenging tasks, for which SOA accuracies are below $90\%$, our method outperforms all the previous methods. Numbers reported in the corresponding papers are used except DANN for which reported scores from \cite{shu2018dirt} are used.

Works we compare to include \cite{ghifary2016deep} which proposed Deep Reconstruction Classification network (DRCN). The cross-entropy loss on the source data and reconstruction loss on the target data are minimized. \cite{shu2018dirt} applied the SSL method VAT to UDA which they call VADA (Virtual Adversarial Domain Adaptation). After training with domain-adversarial loss of \cite{ganin2014unsupervised} and VAT, they further fine-tune only on the target data with the entropy and VAT objectives. \cite{kumar2018co} suggested having two hypotheses in a way that they learn diverse feature embeddings while class predictions are encouraged to be consistent. They build this method on VADA of \cite{shu2018dirt}. The proposed method can be further improved by combining with Co-DA even though we ignore it to highlight the effectiveness of the clean method. Compared to Co-DA, our method has the memory and computational time advantage of not training multiple encoders. \cite{saito2017asymmetric} introduced ATT where two networks are trained on the source data and predictions of the networks are used as pseudo labels on the target data. Another network is trained on the target data with pseudo labels. A pseudo label is assigned if two networks agree and at least one of them is confident. 

Source-only models are also reported as baselines. These models are trained without exploiting the target training data in standard supervised learning setting using the same learning procedure (e.g. network, number of iterations etc.) as UDA methods. Since CIFAR has a large labeled set ($45000$ after removing samples of class frog), CIFAR $\rightarrow$ STL has a high accuracy even without exploiting the unlabeled data. Still, the proposed method outperforms the source-only baseline by $2.24 \%$. The target-only models are trained only on the target domain with class labels revealed. The target-only performance is considered as the empirical upper bound in some papers, but it is not necessarily the case, as seen in the CIFAR $\rightarrow$ STL setting where the target data is scarce; thus the target-only model is even worse than the source-only model. 

The advantage of the proposed method is more apparent in the converse direction STL $\rightarrow$ CIFAR where the accuracy increases from $65.44 \%$ to $77.76 \%$. STL contains a very small ($4500$ after removing samples of class monkey) labeled training set. That is why DIRT-T which fine-tunes on the target data, gave unreliable results for CIFAR $\rightarrow$ STL so they only report VADA result.

The source-only baseline has its lowest score in the MNIST $\rightarrow$ SVHN setting. This is a challenging task as MNIST is greyscale, in contrast to color digits in SVHN. Moreover, SVHN contains multiple digits within an image while MNIST pictures contain single, centered digits. SVHN $\rightarrow$ MNIST is a much simpler experimental setting where SOA accuracies are above $99\%$. We achieve SOA in MNIST $\rightarrow$ SVHN while being second best in SVHN $\rightarrow$ MNIST after Co-DA. Note that our accuracy in SVHN $\rightarrow$ MNIST is $99.33\%$. MNIST $\rightarrow$ MNIST-M and SYN-DIGITS $\rightarrow$ SVHN are other saturated tasks where our method beats SOA in the former one while being second best in the latter. At these levels of saturation of the dataset, top-rated performance is not as informative. 

In MNIST $\rightarrow$ SVHN, our method ($89.19\%$) is substantially better than VADA+IN ($73.3\%$) which also uses input smoothing but with DANN (marginal alignment). Similary, in STL $\rightarrow$ CIFAR, VADA achieves $73.5\%$ while our method is SOA with $77.76\%$ accuracy. This shows the effectiveness of our joint-alignment method. 
 
\begin{figure}
\includegraphics[width=9cm]{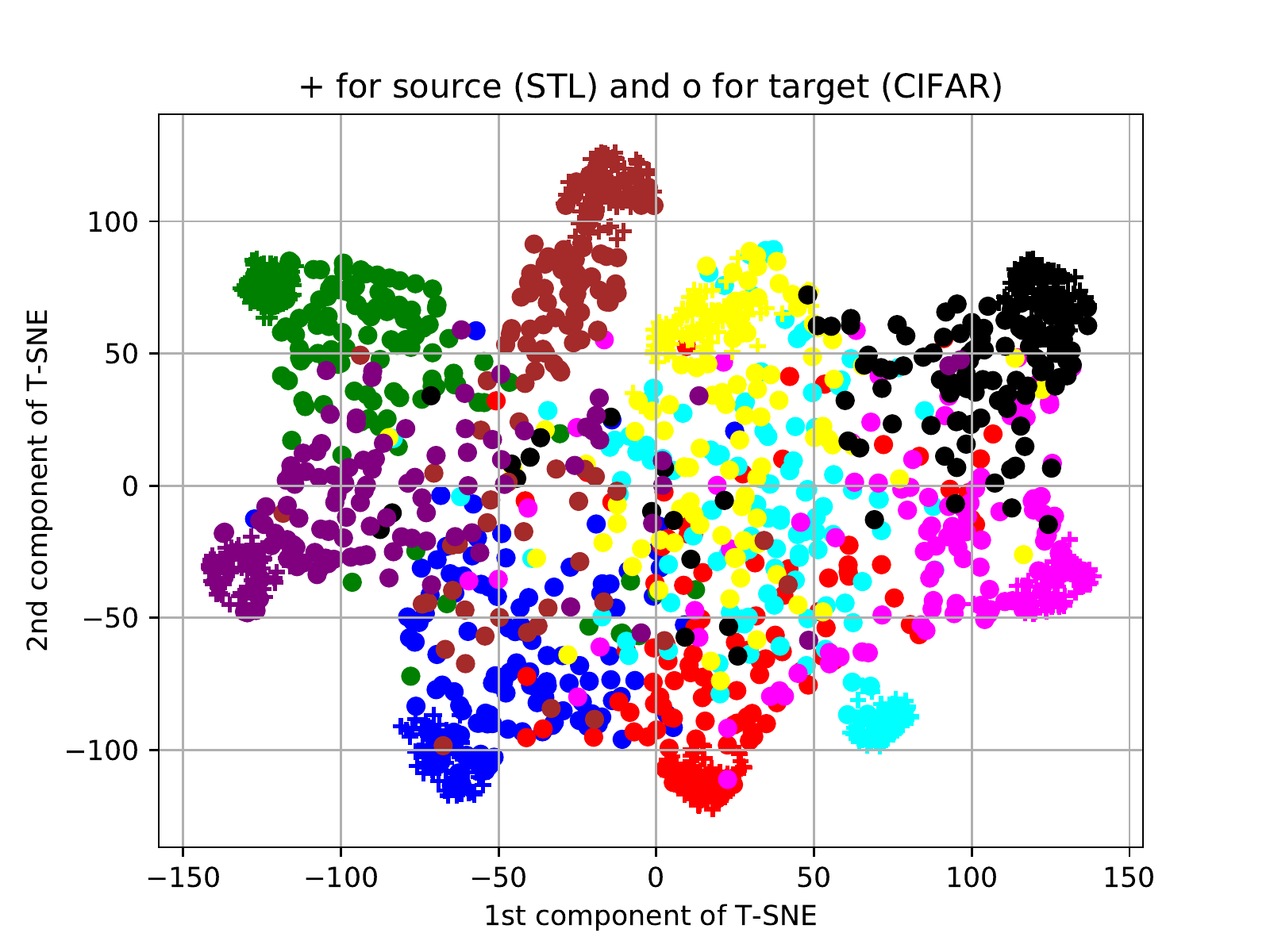}
\includegraphics[width=9cm]{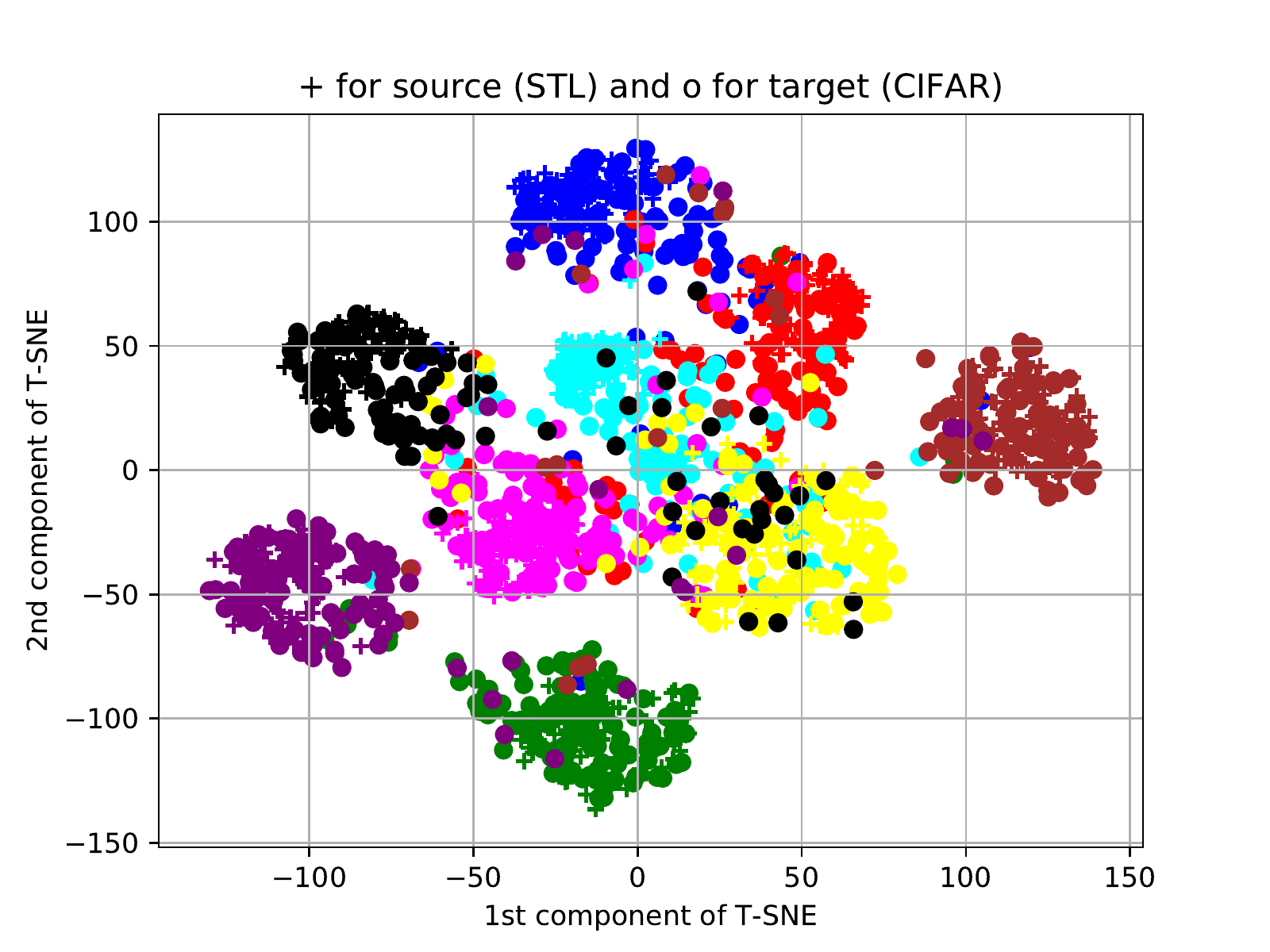}
\caption{{\bf t-SNE plots for STL $\rightarrow$ CIFAR.} t-SNE plots of the source-only trained (top panel) and the proposed method model (bottom panel). Encoder outputs are projected to two-dimensional space with t-SNE. Samples corresponding to the same class are visualized with the same color. The symbol ``+'' is used for the source samples and ``o'' is for the target samples. Best viewed in color.} 
\label{figure_tsne_stl_cifar}
\end{figure}

To demonstrate the effectiveness of the proposed approach in aligning the samples of the same class, we visualize the t-Distributed Stochastic Neighbor Embedding (t-SNE) \cite{maaten2008visualizing} of the source-only baseline and the proposed approach in Fig. \ref{figure_tsne_stl_cifar}. t-SNE is performed on the encoder output for $1000$ randomly drawn samples from both source and target domains for STL $\rightarrow$ CIFAR setting. As one can see, samples of the same classes are better aligned for the proposed approach compared to the source-only method. 

\section{Analysis}
\label{section_analysis}

The main result of our analysis is that the objective introduced in Sect. \ref{section_loss} is minimized only for matching conditional push-forwards given the optimal joint predictor (Theorem \ref{optimal_generated_dist}). For that, we first find the optimal joint predictor in Proposition \ref{optimal_classifier}. We operate  under the supervised setting, assuming the target labels are revealed. So, we replace $\hat y$ in the objective functions with ground-truth labels $y$ for the target samples. Proofs follow similar steps to Proposition 1 and Theorem 1 in \cite{goodfellow2014generative}. 

\begin{prop}
\label{optimal_classifier}
The optimal joint predictor $h_j$ minimizing $L_{jsc}(h_{j}) + L_{jtc}(h_{j})$ given in the Eq. \ref{loss_jsc},\ref{loss_jtc} for any feature $z$ with non-zero measure either on $g\#P_x^s(z)$ or $g\#P_x^t(z)$ is
\begin{align}
h_j(z)[i] = \frac{g\#P_x^s(z, y=e_i)}{g\#P_x^s(z) + g\#P_x^t(z)} \nonumber 
\end{align}
\begin{align}
h_j(z)[i+K] = \frac{g\#P_x^t(z, y=e_i)}{g\#P_x^s(z) + g\#P_x^t(z)} \text{ for } i \in \{1,...,K \} \nonumber
\end{align}
\end{prop}

\begin{theorem}
The objective $L_{jsa}(g) + L_{jta}(g)$ given in the Eq. \ref{loss_jsa}-\ref{loss_lcta} is minimized for the given optimal joint predictor if only if $g\#P_x^s(z|y=e_k)=g\#P_x^t(z|y=e_k)$ and $g\#P_x^s(z|y=e_k)>0 \Rightarrow g\#P_x^s(z|y=e_i)=0$ for $i \neq k$ for any $y=e_k$ and $z$. 
\label{optimal_generated_dist}
\end{theorem}

Theorem \ref{optimal_generated_dist} states that no two samples with different labels can be assigned to the same feature point for the encoder to minimize its loss given the optimal joint predictor. Moreover, the measure assigned to each feature is same for the source and the target push-forward distributions to maximally fool the optimal joint predictor. 

This result indicates that the global minimum of the proposed objective function is achieved when conditional feature distributions are aligned. But, this analysis does not necessarily give a guarantee that the converged solution is optimal in practice as we do not have access to the target labels in UDA. But, we demonstrated empirically in Fig. \ref{figure_tsne_stl_cifar} that with reasonably good pseudo-labels provided by a separate class predictor, the objective gives better alignment than the source-only model. 

The second issue is that finding the optimal predictor or generator with finite samples may not be possible, as optimal solutions are derived as functions of true measures instead of the network parameters trained on finite samples. Lastly, the joint predictor is not trained until convergence; instead, a gradient step is taken in alternating fashion for computational efficiency. So, the predictor is also not necessarily optimal in practice. Even though there are still gaps to be filled between this theory and practice, this analysis shows us that the proposed objective function is doing a sensible job given pseudo-labels for the target data are reasonably good.

\section{Discussion and Related Work}
\label{section_discussion}
In this section, we will summarize the most relevant works from the UDA literature. For more in-depth coverage of the literature see the recent survey of \cite{wang2018deep} on deep domain adaptation for various vision tasks. Many of the domain adaptation works can be categorized into two: (1) the ones learning a shared feature space (symmetric feature based) and the ones transferring features of one domain to another (asymmetric feature based). 

\textbf{Shared feature (symmetric feature based).} Feature transferability drops in the higher layers of a network and there may not exist an optimal classifier for both the source and the target data. Hence many works use two separate classifiers for the the source and the target domains while the encoder parameters are shared. In these works, the source classifier is trained with the labeled source data and the target classifier is regularized by minimizing a distance metric between the source classifier using all the data. 

One common such metric is the (Maximum Mean Discrepancy) MMD which is a measurement of the divergence between two probability distributions from their samples by computing the distance of mean embeddings: $|| \frac{1}{N^s} \sum_{i=1}^{N^s} g(x_i^s) - \frac{1}{N^t} \sum_{i=1}^{N^t} g(x_i^t)||$. DDC of \cite{tzeng2014deep} applies MMD to the last layer while Deep Adaptation Network (DAN) of \cite{long2015learning} applies to the last 3 FC layers. CoGAN of \cite{liu2016coupled} shares early layer parameters of the generator and later layer parameters of the discriminators instead of minimizing the MMD. \cite{long2016unsupervised} models target classifier predictions as the sum of source classifier predictions and a learned residual function. Central Moment Discrepancy (CMD) of \cite{zellinger2017central} extends MMD by matching higher moment statistics of the source and the target features. 

Adversarial domain adaption methods described in the early sections \cite{ganin2014unsupervised} are another way of learning a shared feature space without needing separate classifiers for the source and the target data. DANN \cite{ganin2014unsupervised} proposed a shared encoder and two discriminator branches for domain and class predictions. This makes \textit{marginal} feature distributions similar for the domain classifier.  Upcoming works \cite{shu2018dirt,kumar2018co} applied the same idea but instead of multiplying the gradient with a negative value, they optimize the discriminator and generator losses in an alternating fashion. \cite{shen2017adversarial} suggested replacing the domain discrepancy loss with the Wasserstein distance to tackle gradient vanishing problems. The work of \cite{long2018conditional} resembles ours where they also condition the domain alignment loss to labels. Unlike us, their domain discriminator takes the outer product of the features and the class predictions as input. Similarly, \cite{chen2017no} applies conditional domain alignment using $K$ different class-conditioned binary predictors instead of one predictor with $2K$-way adversarial loss. Our approach allows to not only align the conditional push-forward distributions, but also encourage them to be disjoint. If our sole goal was to align the conditional distributions, a constant encoder function would be a trivial solution. Furthermore, these methods do not exploit SSL regularizers like VAT.

\textbf{Multiple hypotheses.} Another line of work trains multiple encoders and/or classifiers with some consistency loss connecting them. Other than aforementioned methods of \cite{saito2017asymmetric,kumar2018co}, \cite{bousmalis2016domain} proposed domain separation network (DSN). They have two private encoders and a shared encoder for the source and the target samples. The classifier is trained with the summed representations of the shared and the private features. Similarly, \cite{tzeng2017adversarial} trained two encoders for the source and the target data. At test time, they use the encoder learned for the target data and the classifier trained with the source data. \cite{saito2018maximum} had one encoder and two classifiers. Both classifiers are trained on the labeled source samples. The distance between predictions of two classifiers on the same target sample is minimized by the encoder and maximized by classifiers. With the adversarial training of the encoder, they make sure that no two classifiers can have different predictions on the same target sample. Our model also has two predictors but unlike these methods, the purpose of the second predictor (the joint predictor) is to provide conditional alignment for the encoder.

\textbf{Mapping representations (asymmetric feature based).} These methods apply a transformation from the source domain to the target domain or vice-versa \cite{bousmalis2017unsupervised}. Adaptive Batch Normalization (Ad-aBN) of \cite{li2016revisiting} proposed to map domain representations with first-order statistics. Before inference time, they pass all target samples through the network to learn the mean and the variance for each activation and apply these learned statistics to normalize the test instances. \cite{sun2016return} proposed Correlation Alignment (CORAL). They match the second order statistics of the source data to target by recoloring whitened source data with target statistics. 

\textbf{Reconstruction as an auxiliary task.} Another line of work uses reconstruction as an auxiliary task for UDA as in \cite{ghifary2016deep,lee2018diverse}. \cite{zhu2017unpaired} tackles image to image translation (I2I) when there are no paired images in training data. Along with standard GAN losses, they introduced the cycle loss where generators minimize the reconstruction loss. \cite{russo2017source} proposed to modify the consistency loss so that the label of the reconstructed image is preserved, instead of the image itself. \cite{murez2018image} combined several of these reconstruction losses. We have not employed a reconstruction loss as our main focus is domain alignment, not image transfer. 

\textbf{Exploiting unlabeled data with SSL regularizers.} Given the features of the source and the target domains are aligned, standard SSL methods can be applied. \cite{french2018self} employed the Mean Teacher \cite{tarvainen2017mean} for UDA where the consistency loss on the target data between student and teacher networks is minimized. Even with extra tricks like confidence-thresholding and some data augmentation, the accuracy they achieved for MNIST$\rightarrow$SVHN was $34\%$. This shows that, especially when domain discrepancy is high, SSL regularizers are not sufficient without first reducing the discrepancy. 

\textbf{Conditional GAN.} \cite{mirza2014conditional} proposed conditional GAN where generation and discrimination are conditioned onto labels by inputting labels. \cite{odena2017conditional}, instead, augmented the discriminator with an auxiliary task of predicting the class labels. The generator also generates samples respecting the correct class label. Our approach differs from these works as we are not generating fake sample in the input space.

\textbf{Segmentation.} Several works have applied ideas from UDA to semantic segmentation. \cite{zhang2017curriculum} followed the curriculum learning approach, and learn image labels, superpixel labels, and pixel labels in order. \cite{vu2018advent} minimized entropy on the target data in addition to adversarial feature adaptation. \cite{romera2018train} exploited geometric and texture augmentations for domain adaptation. \cite{hoffman2016fcns} was the first one to apply category-specific alignments in the form of lower and upper-bound constraints, but this does not guarantee alignment of the conditional push-forwards. \cite{tsai2018learning} applied the domain adversarial loss on both the features and segmentation outputs. 

\section{Conclusion} 
We proposed a novel method for UDA with the motivation of conditionally aligning the features. We achieved this goal by introducing an additional joint predictor which learns a distribution over class and domain labels. The encoder is trained to fool this predictor within the same-class samples of each domain. We also employed recent tools from SSL to improve the generalization. The proposed idea achieved state-of-the-art accuracy in most challenging image classification tasks for which accuracy are still below $90\%$. The code will be made available after the review process. Implementation details and proofs are provided in the Supp. Mat. 

{\small
\bibliographystyle{ieee}
\bibliography{ref}
}

\onecolumn

\section{Supplementary Material}
\maketitle

In Section \ref{sect_analysis}, we provide proofs for the results given in Section 4 of the main paper. In Section \ref{sect_implementation}, we describe the implementation details. In Section \ref{sect_ablations}, we report and discuss the performance of the proposed method when one or more components in the loss are removed. 

\subsection{Analysis}
\label{sect_analysis}
\setcounter{prop}{0}
\setcounter{defn}{0}
\setcounter{theorem}{0}

First, we prove a simple lemma that will be handy in the proof of Proposition 1. 
\begin{lemma}
\label{lemma_convex}
\begin{align}
\theta^* = \argmin_{\theta} \sum_{i=1}^K - \alpha[i] \log(\theta[i]) \text{ s.t. } 1 \geq \theta[i] \geq 0,\, \sum_{i=1}^K \theta[i] = 1,\, \alpha[i] > 0, \text{ for all } i. \text{ Then, } \theta^*[k] =  \frac{\alpha[k]}{\sum_{i=1}^K \alpha[i]} \text{ for any } k.\nonumber
\end{align}
\end{lemma}
\begin{proof}  
Let us write the Lagrangian form excluding inequality constraints, $L(\theta, \lambda) = \sum_{i=1}^K - \alpha[i] \log(\theta[i]) + \lambda (\sum_{i=1}^K \theta[i] - 1)$. $\nabla_{\theta[i]} L(\theta, \lambda) = -\frac{\alpha[i]}{\theta[i]} + \lambda = 0$ and $\theta[i] = \frac{\alpha[i]}{\lambda}$ for all $i$. $-\log$ is convex hence sum of them also convex and the stationary point is global minima. Then, the dual form becomes $g(\lambda) = \sum_{i=1}^K -\alpha[i] \log(\frac{\alpha[i]}{\lambda}) + \lambda (\sum_{i=1}^K \frac{\alpha[i]}{\lambda} - 1)$ then $\nabla_{\lambda} g(\lambda) = 0$ when $\lambda=\sum_{i=1}^K \alpha[i]$ and $\theta[k]=\frac{\alpha[k]}{\sum_{i=1}^K \alpha[i]}$. Note that the constraint $1 \geq \theta[i] \geq 0$ does not constrain the solution space as $\theta[i]$ has to be non-negative for $\log(\theta[i])$ to be defined and $\sum_{i=1}^K \theta[i] = 1$ enforces $1 \geq \theta[i]$. 
\end{proof}

\begin{prop}
\label{optimal_classifier}
The optimal joint predictor $h_j$ minimizing $L_{jsc}(h_{j}) + L_{jtc}(h_{j})$ given in the Eq. 6,7 for any feature $z$ with non-zero measure either on $g\#P_x^s(z)$ or $g\#P_x^t(z)$ is
\begin{align}
h_j(z)[i] = \frac{g\#P_x^s(z, y=e_i)}{g\#P_x^s(z) + g\#P_x^t(z)} \text{ and } 
h_j(z)[i+K] = \frac{g\#P_x^t(z, y=e_i)}{g\#P_x^s(z) + g\#P_x^t(z)} \text{ for } i \in \{1,...,K \} \nonumber 
\end{align}
\end{prop}

\begin{proof}
\begin{align}
L_{jsc}(h_{j}) + L_{jtc}(h_{j}) \nonumber \\
= E_{(x,y) \sim P^s} \ell_{CE}(h_j(g(x)), [y,\, \mathbf{0}]) + E_{(x,y) \sim P^t} \ell_{CE}(h_j(g(x)), [\mathbf{0},\, y]) \\
= \int_{(x,y) \sim P^s} P_x^s(x) \ell_{CE}(h_j(g(x)), [y,\, \mathbf{0}]) dx +
\int_{(x,y) \sim P^t} P_x^t(x) \ell_{CE}(h_j(g(x)), [\mathbf{0},\, y]) dx \\
= \int_{z \sim g\#P_x^s} \int_{(x,y) \sim P^s s.t. z= g(x)} P_x^s(x) \ell_{CE}(h_j(z), [y,\, \mathbf{0}]) dx dz \nonumber \\
+ \int_{z \sim g\#P_x^t} \int_{(x,y) \sim P^t s.t. z= g(x)} P_x^t(x) \ell_{CE}(h_j(z), [\mathbf{0},\, y]) dx dz \\
= \int_{z \sim g\#P_x^s}\int_{(x,y) \sim P^s s.t. z= g(x)} P_x^s(x) \langle -\log h_j(z), [y,\, \mathbf{0}] \rangle dx dz \nonumber \\
+ \int_{z \sim g\#P_x^t} \int_{(x,y) \sim P^t s.t. z= g(x)} P_x^t(x) \langle -\log h_j(z), [\mathbf{0},\,  y] \rangle dx dz \\
= \int_{z \sim g\#P_x^s} \langle -\log h_j(z),  [\int_{(x,y) \sim P^s s.t. z= g(x)} P_x^s(x) y dx,\, \mathbf{0}] \rangle  dz \nonumber \\
+ \int_{z \sim g\#P_x^t} \langle -\log h_j(z), [\mathbf{0},\,  \int_{(x,y) \sim P^t s.t. z= g(x)} P_x^t(x) y dx] \rangle dz \\
= \int_{z \sim g\#P_x^s} \sum_{i=1}^K -\log h_j(z)[i] g\#P_x^s(z, y=e_i) dz + \int_{z \sim g\#P_x^t} \sum_{i=1}^{K} -\log h_j(z)[i+K] g\#P_x^t(z, y=e_i) dz
\end{align} From Lemma \ref{lemma_convex}, $h_j(z)[i] = \frac{g\#P_x^s(z, y=e_i)}{Z} $ and $h_j(z)[i+K] = \frac{g\#P_x^t(z, y=e_i)}{Z} \text{ for } i \in \{1,...,K \}$ where $Z = \sum_{i=1}^K (g\#P_x^s(z, y=e_i) + g\#P_x^t(z, y=e_i))= g\#P_x^s(z) + g\#P_x^t(z)$ for any $z$. Note that in Lemma \ref{lemma_convex}, we assumed $\alpha[i]>0$ while here $g\#P_x^s(z, y=e_i)$ and $g\#P_x^t(z, y=e_i)$ might be zero for some $i$. But since we are taking $\alpha[i] \log(\theta[i])$ as zero whenever $\alpha[i]=0$ for any value of $\theta[i]$, the result does not change. 
\end{proof}

The following Lemma will be used in the proof of Theorem 1.

\begin{lemma}
$\min_{P, Q} L(P, Q) = E_{x \sim P} -\log \frac{Q(x)}{P(x) + Q(x)} + E_{x \sim Q}  -\log \frac{P(x)}{P(x) + Q(x)}$
is achieved only if $P(x) = Q(x)$ for all $x$.
\label{lemma_pq}
\end{lemma}
\begin{proof}
\begin{align}
L(P, Q) \\
= \int_{x} - P(x) \log(\frac{Q(x)}{P(x)+Q(x)}) - Q(x) \log(\frac{P(x)}{P(x)+Q(x)}) dx \\
= \int_{x} P(x) \log(\frac{P(x)+Q(x)}{Q(x)}) + Q(x) \log(\frac{P(x)+Q(x)}{P(x)}) dx \\
= \int_{x} P(x) \log(1+\frac{P(x)}{Q(x)}) + Q(x) \log(1+\frac{Q(x)}{P(x)}) dx \\
= \int_{x} \log(1+\frac{P(x)}{Q(x)}) (1+\frac{P(x)}{Q(x)})Q(x)  -\log(1+\frac{P(x)}{Q(x)})Q(x) +  Q(x) \log(1+\frac{Q(x)}{P(x)}) dx \\
= \int_{x} \Big( \log(1+\frac{P(x)}{Q(x)}) (1+\frac{P(x)}{Q(x)}) -\log(1+\frac{P(x)}{Q(x)}) + \log(1+\frac{Q(x)}{P(x)}) \Big) Q(x) dx \\
= \int_{x} \big( \log(1+\frac{P(x)}{Q(x)}) (1+\frac{P(x)}{Q(x)}) + \log(\frac{Q(x)}{P(x)}) \big) Q(x) dx \\
= \log(4) - \int_{x} \log(2) \frac{P(x)+Q(x)}{Q(x)} Q(x) dx + \int_{x} \big( \log(1+\frac{P(x)}{Q(x)}) (1+\frac{P(x)}{Q(x)}) + \log(\frac{Q(x)}{P(x)}) \big) Q(x) dx\\
= \log(4) + \int_{x} \big( \log(1+\frac{P(x)}{Q(x)}) (1+\frac{P(x)}{Q(x)}) - \log(\frac{P(x)}{Q(x)}) - \log(2) (1+\frac{P(x)}{Q(x)}) \big) Q(x) dx
\end{align}
Let $\phi(\beta) := \log(1+\beta) (1+\beta) - \log(\beta) - \log(2) (1+\beta)$. Then, $\nabla_{\beta} \phi(\beta) =  1 + \log(1+\beta) - \frac{1}{\beta} - \log(2)$ and $\nabla_{\beta} \nabla_{\beta} \phi(\beta) =  \frac{1}{1+\beta} + \frac{1}{\beta^2} > 0$ for $\beta > 0$. Hence $\phi(\beta)$ is convex and we can apply Jensen,
\begin{align}
L(P, Q) = \log(4) + \int_{x} \phi(\frac{P(x)}{Q(x)}) Q(x) dx \geq  \log(4) + \phi(\int_{x} \frac{P(x)}{Q(x)} Q(x) dx) = \log(4) + \phi(1)  = \log(4)
\end{align}
Since $\phi$ is strictly convex, equality is satisfied only for constant argument i.e. when $\frac{P(x)}{Q(x)}=1$ which is also the global minima of $\phi(\beta)$ as $\nabla_{\beta} \phi(1)=0$.
\end{proof}

\begin{theorem}
The objective $L_{jsa}(g) + L_{jta}(g)$ given in the Eq. 8,9 is minimized for the given optimal joint predictor if only if $g\#P_x^s(z|y=e_k)=g\#P_x^t(z|y=e_k)$ and $g\#P_x^s(z|y=e_k)>0 \Rightarrow g\#P_x^s(z|y=e_i)=0$ for $i \neq k$ for any $y=e_k$ and $z$. 
\label{optimal_generated_dist}
\end{theorem}

\begin{proof}
The objective for encoder is,
\begin{align}
E_{(x,y) \sim P^s} \ell_{CE}((h_j(g(x)), [\mathbf{0},\, y]) + E_{(x,y) \sim P^t} \ell_{CE}((h_j(g(x)), [y,\,  \mathbf{0}])
\end{align}
For samples with label $e_k$ we want to minimize,
\begin{align}
E_{(x,y) \sim P^s(x,y=e_k)} \ell_{CE}((h_j(g(x)), [\mathbf{0},\, y]) 
+ E_{(x,y) \sim P^t(x,y=e_k)} \ell_{CE}((h_j(g(x)), [y,\,  \mathbf{0}]) \\
= E_{(x,y) \sim P^s(x,y=e_k)}  - \langle [\mathbf{0},\, y], \log (h_j(g(x))\rangle +E_{(x,y) \sim P^t(x,y=e_k)}   - \langle [y,\, \mathbf{0}], \log (h_j(g(x))\rangle  \\
= -E_{(x,y) \sim P^s(x,y=e_k)} \log (h_j(g(x))[k+K]-E_{(x,y) \sim P^t(x,y=e_k)}  \log (h_j(g(x))[k] 
\end{align}
Given the classifier $h_j$ is optimal, the above them becomes
\begin{align}
- \int_{z \sim g\#P_x^s(z,y=e_k)} g\#P_x^s(z,y=e_k) \log \frac{g\#P_x^t(z,y=e_k)}{\sum_{i=1}^K (g\#P_x^s(z,y=e_i) + g\#P_x^t(z,y=e_i)} dz \nonumber \\
- \int_{z \sim g\#P_x^t(z,y=e_k)}  g\#P_x^t(z,y=e_k) \log \frac{g\#P_x^s(z,y=e_k)}{\sum_{i=1}^K (g\#P_x^s(z,y=e_i) + g\#P_x^t(z,y=e_i))} dz \\
= \int_{z \sim g\#P_x^s(z,y=e_k)} g\#P_x^s(z,y=e_k) \Big( -\log \frac{g\#P_x^t(z,y=e_k)}{g\#P_x^s(z,y=e_k)+g\#P_x^t(z,y=e_k)} \nonumber \\
+ \log \frac{\sum_{i=1}^K (g\#P_x^s(z,y=e_i) + g\#P_x^t(z,y=e_i))}{g\#P_x^s(z,y=e_k) + g\#P_x^t(z,y=e_k)}\Big) dz \nonumber \\
+ \int_{z \sim g\#P_x^t(z,y=e_k)} g\#P_x^t(z,y=e_k) \Big( -\log \frac{g\#P_x^s(z,y=e_k)}{g\#P_x^s(z,y=e_k)+g\#P_x^t(z,y=e_k)}  
\nonumber \\
+ \log \frac{\sum_{i=1}^K (g\#P_x^s(z,y=e_i) + g\#P_x^t(z,y=e_i))}{g\#P_x^s(z,y=e_k) + g\#P_x^t(z,y=e_k)}\Big) dz 
\end{align}
Let us write first and second terms in each integration separately:
\begin{align}
L_1(g\#P_x^s, g\#P_x^t)=- \int_{z \sim g\#P_x^s(z,y=e_k)} g\#P_x^s(z,y=e_k) \log \frac{g\#P_x^t(z,y=e_k)}{g\#P_x^s(z,y=e_k)+g\#P_x^t(z,y=e_k)} dz \nonumber \\
- \int_{z \sim g\#P_x^t(z,y=e_k)} g\#P_x^t(z,y=e_k) \log \frac{g\#P_x^s(z,y=e_k)}{g\#P_x^s(z,y=e_k)+g\#P_x^t(z,y=e_k)} dz
\end{align}
\begin{align}
L_2(g\#P_x^s, g\#P_x^t)= \int_{z \sim g\#P_x^s(z,y=e_k)} g\#P_x^s(z,y=e_k) \log  \frac{\sum_{i=1}^K (g\#P_x^s(z,y=e_i) + g\#P_x^t(z,y=e_i))}{g\#P_x^s(z,y=e_k) + g\#P_x^t(z,y=e_k)} dz \nonumber \\
+ \int_{z \sim g\#P_x^t(z,y=e_k)} g\#P_x^t(z,y=e_k) \log  \frac{\sum_{i=1}^K (g\#P_x^s(z,y=e_i) + g\#P_x^t(z,y=e_i))}{g\#P_x^s(z,y=e_k) + g\#P_x^t(z,y=e_k)} dz 
\end{align}
If there is a solution which is global minima of both $L_1$,$L_2$ then it is also the global minima of the overall term $L_1+L_2$. $L_2$ has its minimum at $\sum_{i=1 s.t. i \neq k}^K (g\#P_x^s(z,y=e_i) + g\#P_x^t(z,y=e_i))=0$ whenever $g\#P_x^s(z,y=e_k) + g\#P_x^t(z,y=e_k)>0$. From Lemma \ref{lemma_pq}, $L_1(g\#P_x^s, g\#P_x^t)$ achieves its minimum only when $g\#P_x^s(z,y=e_k)=g\#P_x^t(z,y=e_k)$ for any $z$. Intersection of two minimas gives the desired solution.

\end{proof}

\subsection{Implementation Details}
\label{sect_implementation}

Batchsize of $64$ is used for both the source and the target samples during training. Batchsize of $100$ is used at inference time. The networks used in the experiments are given in Table \ref{networks}. Instance norm is only used in MNIST $\leftrightarrow$ SVHN experiments. In all experiments, networks are trained for $60,000$ iterations. This is less than $80,000+80,000=160,000$ iterations that SOA methods VADA+DIRT-T and Co-DA+DIRT-T are trained for. Weight decay of $10^{-4}$ is used. In CIFAR $\leftrightarrow$ STL and SYN-DIGITS $\rightarrow$ SVHN, as an optimizer we use SGD with the initial learning rate of $0.1$. Learning rate is decreased to $0.01$ at iteration $40,000$. Momentum of SGD is $0.9$. In MNIST $\leftrightarrow$ SVHN and MNIST $\rightarrow$ MNIST-M, Adam optimizer with the fixed learning rate $0.001$ is used. Momentum is chosen to be $0.5$.

We fix $\lambda_{t} = 0.1$ and  $\lambda_{jsc}=1.0$. We searched rest of the parameters over $\lambda_{tvat}\in \{1.0,10.0\}$, $\lambda_{jtc} \in \{1.0,10.0\}$, $\lambda_{jta} \in \{0.1,1.0\}$, $\lambda_{svat}=\{0.0,1.0\}$, $\lambda_{jsa} \in \{0.1,1.0\}$. We also searched for the upper bound of the adversarial perturbation in VAT, $\epsilon_{x} \in \{0.1,0.5,1.0,2.0,4.0,8.0\}$. Optimal hyperparameters are given in Table \ref{hyperparams_table} for each task. Only for MNIST $\rightarrow$ SVHN, class predictor performs poorly in the early epochs. So, we apply curriculum learning within  $60,000$ iterations. In the first $4,000$ iterations, only $\lambda_{jsc}$ and $\lambda_{jsa}$ are non-zero i.e. losses only depending on the labeled-source data are minimized. After $4,000$ iterations, SSL regularizations are started to be applied: $\lambda_{t}$, $\lambda_{svat}$ and $\lambda_{tvat}$ are also set to non-zero. After $8,000$ iterations, losses depending on the pseudo-labels are activated by assigning all hyperparameters to their optimal values given in Table \ref{hyperparams_table}.

\begin{table}
\begin{center}
\begin{tabular}{||c | c | c | c | c | c | c ||} 
\hline
Source dataset & MNIST  & SVHN  & CIFAR  & STL    & SYN-DIGITS & MNIST   \\ [0.5ex]
Target dataset & SVHN   & MNIST & STL    & CIFAR  & SVHN       & MNIST-M \\ [0.5ex]
\hline\hline
$\lambda_{t}$ & 0.1 & 0.1 & 0.1 & 0.1 & 0.1 & 0.1 \\
\hline
$\lambda_{tvat}$ & 10.0 & 10.0 & 10.0 & 10.0 & 10.0 & 10.0 \\
\hline
$\lambda_{jtc}$ & 10.0 & 1.0 & 1.0 & 1.0 & 10.0 & 10.0\\
\hline
$\lambda_{jta}$ & 1.0 & 0.1 & 0.1 & 0.1 & 0.1 & 1.0 \\
\hline
$\lambda_{svat}$ & 0.0 & 0.0 & 0.0 & 0.0 & 1.0 & 0.0 \\
\hline
$\lambda_{jsc}$ & 1.0 & 1.0 & 1.0 & 1.0 & 1.0 & 1.0\\
\hline
$\lambda_{jsa}$ & 1.0 & 0.1 & 1.0 & 1.0 & 1.0 & 1.0 \\
\hline
$\epsilon_{x}$ & 4.0 & 4.0 & 2.0 & 1.0 & 1.0 & 0.5 \\
\hline\hline
\end{tabular}
\caption{{\bf Hyperparameters.} Hyper-parameters used in the proposed method for each task. }
\label{hyperparams_table}
\end{center}
\end{table}

\begin{table}
\begin{tabular}{||l||} 
\hline
\underline{Encoder}\\
$3 \times 3$ convolution, 64 lReLU\\
$3 \times 3$ convolution, 64 lReLU\\
$3 \times 3$ convolution, 64 lReLU\\
$2 \times 2$ max-pool, stride 2, dropout with probability $0.5$\\

$3 \times 3$ convolution, 64 lReLU\\
$3 \times 3$ convolution, 64 lReLU\\
$3 \times 3$ convolution, 64 lReLU\\
$2 \times 2$ max-pool, stride 2, dropout with probability $0.5$\\
\hline
\hline
\underline{Class predictor}\\
$3 \times 3$ convolution, 64 lReLU\\
$1 \times 1$ convolution, 64 lReLU\\
$1 \times 1$ convolution, 64 lReLU\\
Global average pooling, 6 \time 6 $\rightarrow$ 1 \time 1\\
Fully connected layer: 128 $\rightarrow$ K\\
Softmax\\
\hline
\hline
\underline{Joint predictor}\\
$3 \times 3$ convolution, 64 lReLU\\
$1 \times 1$ convolution, 64 lReLU\\
$1 \times 1$ convolution, 64 lReLU\\
Global average pooling, 6 \time 6 $\rightarrow$ 1 \time 1\\
Fully connected layer: 128 $\rightarrow$ 2K\\
Softmax\\
 \hline
\end{tabular}
\quad
\begin{tabular}{||l||} 
\hline
\underline{Encoder}\\
$3 \times 3$ convolution, 128 lReLU\\
$3 \times 3$ convolution, 128 lReLU\\
$3 \times 3$ convolution, 128 lReLU\\
$2 \times 2$ max-pool, stride 2, dropout with probability $0.5$\\

$3 \times 3$ convolution, 256 lReLU\\
$3 \times 3$ convolution, 256 lReLU\\
$3 \times 3$ convolution, 256 lReLU\\
$2 \times 2$ max-pool, stride 2, dropout with probability $0.5$\\

$3 \times 3$ convolution, 512 lReLU\\
$1 \times 1$ convolution, 256 lReLU\\
$1 \times 1$ convolution, 128 lReLU\\
Global average pooling, 6 \time 6 $\rightarrow$ 1 \time 1\\
\hline
\hline
\underline{Class predictor}\\
Fully connected layer: 128 $\rightarrow$ K\\
Softmax\\
\hline
\hline
\underline{Joint predictor}\\
Fully connected layer: 128 $\rightarrow$ 2K\\
Softmax\\
\hline
\hline
\end{tabular}

\caption{\textbf{Left.} The network used in the tasks involving MNIST dataset (i.e. MNIST $\leftrightarrow$ SVHN and MNIST $\rightarrow$ MNIST-M), \textit{small-net} from \cite{shu2018dirt,kumar2018co}. \textbf{Right.} The network used in the rest of the classification tasks (i.e. STL $\leftrightarrow$ CIFAR, SYN-DIGITS $\rightarrow$ SVHN), \textit{conv-large} from \cite{french2018self}. Slope of each leaky RELU (lReLU) layer is $0.1$. Each conv is followed by a batch norm layer. }
\label{networks}
\end{table}

\subsection{Ablations}
\label{sect_ablations}

\begin{table}
\begin{center}
\begin{tabular}{||c | c | c | c | c | c | c ||} 
\hline
Source dataset & MNIST  & SVHN  & CIFAR  & STL    & SYN-DIGITS & MNIST   \\ [0.5ex]
Target dataset & SVHN   & MNIST & STL    & CIFAR  & SVHN       & MNIST-M \\ [0.5ex]
\hline\hline
Without VAT & $60.65$ & $98.79$ & $81.59$ & $70.20$ & $93.15$ & $98.45$ \\
\hline
Without EntMin and VAT & $62.95$ & $88.33$ & $80.97$ & $71.62$ & $92.10$ & $97.74$ \\
\hline
Without source alignment & $75.78$ & $88.91$ & $81.11$ & $74.80$ & $95.72$ & $99.25$ \\
\hline
Without target alignment & $71.59$ & $98.89$ & $80.90$ & $74.87$ & $95.48$ & $99.20$ \\
\hline
Without source and target alignment & $60.07$ & $98.83$ & $80.20$ & $73.52$ & $94.94$ & $99.08$ \\
\hline\hline
Source-only (baseline) & $44.21$ & $70.58$ & $79.41$ & $65.44$ & $85.83$ & $70.28$\\
\hline
The proposed loss with the class predictor & $89.19$ & $99.33$ & $81.65$ & $77.76$ & $96.22$& $99.47$\\
\hline
The proposed loss with the joint predictor & $87.88$ & $99.16$ & $81.19$ & $77.62$ & $95.97$ & $99.40$ \\
\hline\hline
\end{tabular}
\caption{{\bf Ablations.} Performance of the proposed method when one or two terms in the loss function are removed (first five rows). We also report performance of the source-only baseline ($6$th row) and model optimizing the original loss ($7$th row) as a reference. In the last row, we report the performance of the joint predictor.}
\label{ablations_table}
\end{center}
\end{table}

In Table \ref{ablations_table}, we report the performance of the proposed method by removing one or more components from the original loss function. We report the results by removing VAT regularizations ($\lambda_{svat}=\lambda_{tvat}=0$), VAT and entropy minimization, the source-alignment loss  ($\lambda_{jsa}=0$), the target-alignment loss ($\lambda_{jta}=0$) and both alignment losses ($\lambda_{jsa}=\lambda_{jta}=0$). Removing any of these components degraded the performance in all the tasks. All results are still better than the source-only model. 

Removing both entropy-minimization and VAT losses makes the performance worse than only removing VAT losses in all the tasks except STL$\rightarrow$CIFAR and MNIST$\rightarrow$SVHN. In these tasks, the entropy-minimization loss only helped when it combined with VAT losses. This is expected as the entropy minimization without VAT regularizations can easily lead to trivial, degenerate solutions by encouraging to cluster samples from different classes. Removing the source and the target-alignment losses together degraded the performance compared to removing either of them except SVHN$\rightarrow$MNIST. Applying the target-alignment loss without the source-alignment loss might have a detrimental effect as former one relies on the {\em noisy} pseudo-labels.    

We also report the best performances with the joint-predictor for completeness. The joint predictor achieves very close performance to the class-predictor but it is slightly worse than the class-predictor. We believe this is because the joint-predictor is trained for the harder task of domain and class learning while only latter one is needed at the test time. That is why we choose to use the class-predictor for inference.

\begin{table}
\begin{center}
\begin{tabular}{||l l||}
\hline
Notation & Description \\
\hline\hline
$x$ & Input to the network. \\
$z$ & Encoder output. \\
$y$ & 1-hot ground truth label. \\
$K$ & Number of classes. \\
$g(x)$ & The encoder. \\
$h_j(x)$ & The joint predictor. \\
$h_c(x)$ & The class predictor. \\
$f_j(x)=h_j(g(x))$ & Composition of the encoder and the joint predictor. \\
$f_c(x)=h_c(g(x))$ & Composition of the encoder and the class predictor. \\
$P^s(x,y)$ & Joint distribution over source samples and labels. \\
$P^t(x,y)$ & Joint distribution over target samples and labels. \\
$P^s_x(x)$ & Marginal distribution over source samples. \\
$P^t_x(x)$ & Marginal distribution over target samples. \\
$g\#P_x^s(z)$ & Push-forward source distribution. \\
$g\#P_x^t(z)$ & Push-forward target distribution. \\
$N^s$ & Number of source training samples. \\
$N^t$ & Number of target training samples. \\
$X^s$ & Set of source training samples.\\
$X^t$ & Set of target training samples.\\
$\epsilon_x$ & Upper bound on the norm of adversarial input perturbations. \\
$x[k]$ & $k$th value of the vector $x$. \\
 \hline
\end{tabular}
\caption{The notation used in the paper.}
\label{notation}
\end{center}
\end{table}

\end{document}